\documentclass[runningheads]{llncs}

\usepackage{graphicx}
\usepackage{booktabs}
\usepackage{CJK}
\usepackage{booktabs}
\usepackage{multirow}

\usepackage{tabularx}
\usepackage{makecell}
\usepackage{threeparttable}
\usepackage{siunitx}
\usepackage[accsupp]{axessibility}  

\usepackage{amsmath}
\usepackage{amssymb}
\usepackage{bm}
\newtheorem{assumption}{Assumption}
\usepackage{hyperref}
\usepackage{cleveref}
\usepackage{orcidlink}

\begin{document}

\title{Listening makes Vision Clear for VLMs}


\author{Yiyang Chen\inst{1} \and
Yixin Tan\inst{1} \and
Binrui Shen\inst{1}}
\authorrunning{Chen. et al.}

\institute{Beijing Normal University}
\maketitle

\begin{abstract}
Recent work typically assesses vision–language consistency using attention distributions of answer-side tokens. However, we observe that highest attention regions are not always consistent with the intended semantic token. This probably stems from decoding drift, where language priors from previously generated answer tokens accumulate and mismatch with visual attention.
Besides the priors from previous answer tokens, we find that structural tokens (e.g., modality boundary markers) may encompass the entire context and generate high attention to areas unrelated to the target.
To avoid these distortions and provide consistency evaluation for large VLMs, we adopt prompt-side semantics and propose \textbf{P}rompt-\textbf{V}ision \textbf{T}oken \textbf{A}ctivation \textbf{M}ap (PV-TAM).
PV-TAM further incorporates a filter to remove systematic bias induced by modality boundary markers (e.g., \texttt{<vision\_start>}, \texttt{<vision\_end>}). 
Unlike traditional methods that evaluate overlap solely through masks while ignoring activation intensity, our metrics leverage the peak distribution of attention to measure the alignment between prompts and visual regions. 
In experiments, PV-TAM consistently improves both attention-based and IoU-style localization metrics over answer-side baselines on various datasets.

\end{abstract}

\section{Introduction}
\label{sec:intro}

Vision–language models (VLMs) have achieved impressive multimodal capabilities and are increasingly used in tasks that require reliable semantic alignment between text and vision.
Language–vision semantic consistency is crucial for ensuring that VLM predictions are grounded in the right visual area rather than language priors, enabling reliable interpretation, evaluation, and visual grounding \cite{selvaraju2019taking,zhang2022glipv2,kang2025your,zhang2026adavboost}. 
In this work, we use \textbf{consistency} to assess the degree of alignment between a visual region and its corresponding semantic token. 
Most existing approaches estimate such alignment from attention maps of answer tokens during decoding\cite{li2025tam,abnar2020quantifying}.

However, we notice that aligning language and vision solely on the answer side may not accurately reflect consistency. In practice, the highest attention regions are not always consistent with the intended semantic token \cite{wiegreffe-pinter-2019-attention,serrano2019attention}. Take \cref{fig:invalidcase} as an example, given a ''duck neck'' query, attention fails to localize the neck and instead highlights the duck wing region. This misalignment might lie in decoding drift, where language priors from previously generated answer tokens accumulate in the growing context and gradually dominate the prompt–vision relation. 
\begin{figure}[htbp]
    \centering    \includegraphics[width=1.\linewidth]{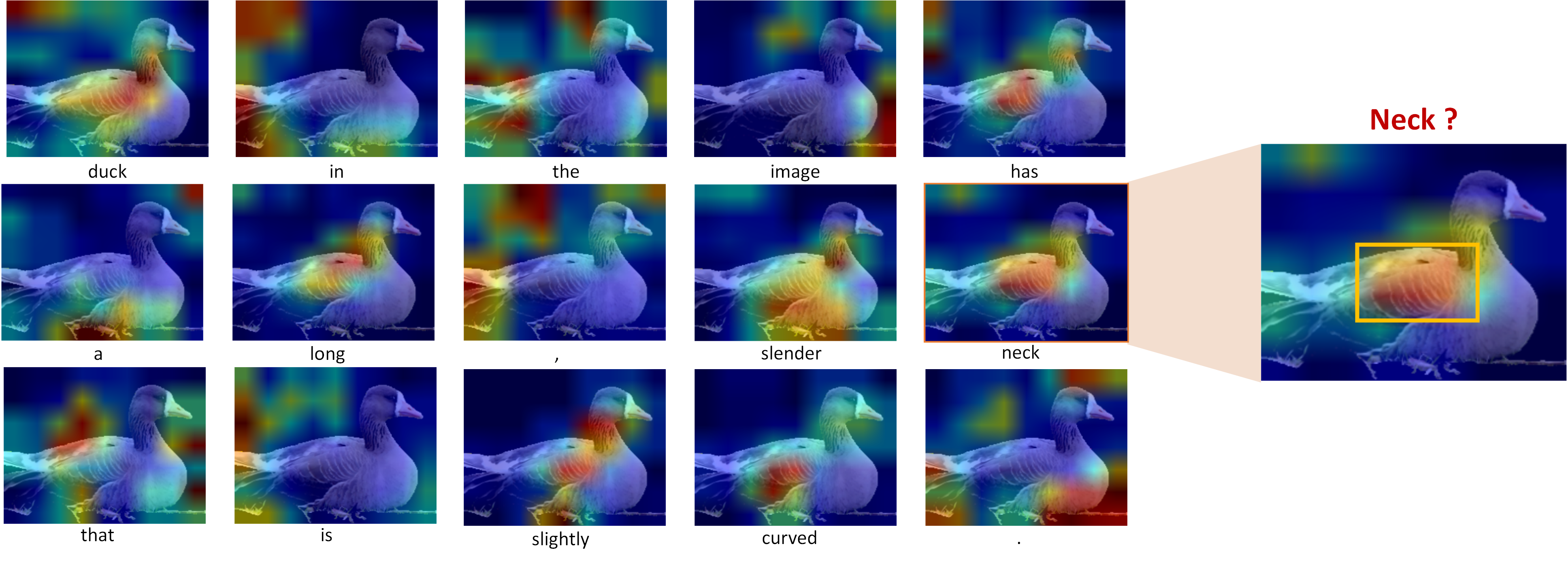}
    \caption{A representative example showing the sensitivity of TAM to preceding context \cite{li2025tam}. The more red indicates higher attention. But  highest attention region is not consistent with the intended semantic area \texttt{neck}. We can see that the input text is a sentence where prior words actually hurt the vision-language alignment.}
\label{fig:invalidcase}
\end{figure}
To avoid this issue, we shift the perspective from answer-side generation to prompt-side semantics, where the queried tokens are fixed inputs.
From a probabilistic perspective, the generated answer content is inherently influenced by prior context \cite{bengio2015scheduled,schmidt2019generalization}. The semantic information of the current output is not independent of context, potentially introducing drift phenomena. Conversely, tokens on the fixed prompt side provide more stable token-to-vision grounding.
Based on these insights, we propose a prompt-based attention map framework that extracts and denoises token-to-vision activation maps from prompt tokens by suppressing structural-token attention, and evaluates localization precision with two metrics: TDR and TGR.
Specifically, we make three contributions:
\begin{itemize}
    \item \textbf{Prompt-Vision Token Activation Map.} We propose a prompt-side framework to address language-vision alignment without relying on answer-side generations to reduce drift and contextual entanglement.
    \item \textbf{Structure-Token Denoising.}
    We identify structural tokens surrounding the target semantic span and show that their activation maps mainly capture structural bias rather than semantics. We justify removing structural-token attention to obtain a denoised token activation map for the target semantics.
    \item \textbf{Alignment-Oriented Evaluation.} 
We designed two localization metrics: TDR to quantify global coverage and TGR to quantify coverage within target semantic regions. These metrics provide a novel perspective for evaluating the accuracy of prompt-to-vision alignment.
\end{itemize}
Together, these components form a framework that turns prompt tokens into a practical diagnostic tool for VLMs in semantic alignment analysis. 
In experiments, our method outperforms or matches prior interpretability baselines (e.g., CAM\cite{jiang2021layercam}, CP-LRP\cite{ali2022xai}, Attention Rollout\cite{abnar2020quantifying}, TAM\cite{li2025tam}) on TGR/TDR and Min-Dist, demonstrating consistently improved prompt-to-vision alignment.

\section{Related Works}
\subsection{Vision-Language Models}
Recent VLMs typically relied on visual and language inputs. After tokenizing and concatenating them, the token sequence was fed into cross-modal Transformers to learn language-visual alignment for tasks such as VQA, retrieval, and referring expression comprehension.
UNITER and OSCAR explicitly modeled token--region matching in multi-task pretraining, and further strengthened cross-modal binding with object tags or alignment objectives~\cite{chen2020uniter,li2020oscar}.
VinVL pushed this line forward by improving the \emph{visual region representations themselves}, demonstrating that stronger detection features lead to substantial gains in downstream cross-modal understanding and localization~\cite{zhang2021vinvl}.

Localization capability has been incorporated more directly into pretraining.
MDETR modulates DETR with text queries, enabling end-to-end text-conditioned detection~\cite{kamath2021mdetr}.
RegionCLIP and GLIP introduce region-level contrastive alignment or joint objectives that combine detection with phrase alignment, improving open-vocabulary transferability~\cite{zhong2022regionclip,li2022glip,zhang2022glipv2,li2022elevater}.
Meanwhile, as ViT-style patch tokens become the dominant visual representation, cross-modal interaction increasingly shifts to the token level, providing a structural basis for analyzing token-wise visual dependence inside modern VLMs.
At the evaluation level, phrase grounding (PG) and referring expression grounding (RECG) still largely rely on IoU-style criteria\cite{chen2025revisiting,kazemzadeh2014referitgame}; datasets such as Flickr30k Entities provide stable phrase--box annotations for entity-to-region evaluation~\cite{plummer2015flickr30kentities}.
However, these evaluations emphasize final mask quality and do not directly capture the consistency of language and vision in VLMs.

\subsection{Attention Interpretability and Evaluation}
A broad literature visualizes attention weights or gradient-based attributions to explain where a model attends. Typical examples include Grad-CAM for gradient-based localization and SmoothGrad for noise-reduced saliency \cite{selvaraju2017grad,smilkov2017smoothgrad}.
However, attention is not necessarily equivalent to causal contribution\cite{wiegreffe-pinter-2019-attention}. Attention weights can be redistributed without significantly changing the output, which may yield plausible but weakly faithful explanations.
In multimodal models, gradient-based localization and perturbation-style validation have been used to provide more testable explanations\cite{Fong_2019_ICCV}, while quantitative protocols for saliency evaluation can be sensitive to implementation choices and exhibit instability\cite{petsiuk2021black}.
Another line of work exploits Transformer structure, like selecting several heads, to obtain more stable attention patterns\cite{kang2025your}. But stability alone does not guarantee that a map faithfully reflects a specific token's visual evidence-especially under autoregressive generation\cite{Chefer_2021_CVPR}.

\subsection{Priors, Hallucination, and Noisy Supervision in VLMs}
In generative VLMs, the output follows an autoregressive mechanism. Each newly generated token is influenced not only by the image but also by semantic priors and accumulated contextual states. Prior studies in VQA suggest that the model may be biased towards utilizing language information rather than fully leveraging multimodal information. As a result, it can still deliver high-confidence supervised task results even with insufficient visual information, while unsupervised generative tasks such as image captioning often exhibit hallucination, generating entities that do not exist in the image\cite{Goyal_2017_CVPR,agrawal2018don,rohrbach-etal-2018-object}.
During model training and prediction, the language-visual mapping built based on generated labels may be amplified due to the dominance of language priors or distorted due to supervision alignment bias\cite{sharma-etal-2018-conceptual,jia2021scaling}.
The problems arising from prior knowledge, hallucination, and noise necessitate better evaluation criteria to assess true language-visual relationship.
\section{Methods}
\subsection{Problem Formulation}
\label{subsec:problem}
We consider Transformer-based VLMs. Given an image $\mathcal{I}$, a vision encoder produces a sequence of visual tokens $\mathcal{T}_v=\{t_i^v\}_{i=1}^{N_v}$, $t_i^v\in\mathbb{R}^d$.
A text prompt $\mathcal{P}$ is tokenized and embedded into prompt tokens $\mathcal{T}_p=\{t_j^p\}_{j=1}^{N_p}$.
At decoding step $k$, the previously generated answer tokens are $\mathcal{T}_a^{<k}=\{t_m^a\}_{m=1}^{k-1}$, and the full token sequence is
$\mathcal{T}_{\le k}=\mathcal{T}_v \oplus \mathcal{T}_p \oplus \mathcal{T}_a^{<k}$.
In layer $l$ and head $h$, let $Q^{(l,h)}=W_Q^{(l,h)}\mathcal{T}_{\le k}$ and $K^{(l,h)}=W_K^{(l,h)}\mathcal{T}_{\le k}$, where $W_Q^{(l,h)}, W_K^{(l,h)} \in \mathbb{R}^{d_k \times d}$ denote the query and key projection matrices that map the token sequence $\mathcal{T}_{\le k}$ to the $d_k$-dimensional latent space of the $h$-th attention head in the $l$-th layer.
The causal self-attention weights are
\begin{equation}
A^{(l,h)}=\mathrm{Softmax}\!\left(\frac{Q^{(l,h)}{K^{(l,h)}}^\top}{\sqrt{d_k}}+M\right)\in\mathbb{R}^{N\times N},
\end{equation}
where $N=N_v+N_p+k-1$ and $M$ is the causal mask whose entries are $0$ for valid positions and $-\infty$ for masked future tokens\cite{vaswani2017attention}. 
For a text token $t$, we extract its attention toward visual tokens as a vector
\begin{equation}
\mathbf{a}^{(l,h)}(t)=\left[A^{(l,h)}_{t,t_1^v},\ldots,A^{(l,h)}_{t,t_{N_v}^v}\right]^\top\in\mathbb{R}^{N_v}.
\end{equation}
Assuming the visual tokens form a flattened 2D grid ($N_v = H'W'$),
we reshape the attention vector $\mathbf{a}^{(l,h)}(t) \in
\mathbb{R}^{N_v}$ into a spatial attention map in
$\mathbb{R}^{H' \times W'}$. The map is then bilinearly upsampled to
the original image resolution to produce a dense token-to-visual
activation map.

\subsection{Explain Inconsistency through Prior Context Interference}
\label{sec:context_interference}

For decoding step $j$, an autoregressive VLM predicts a next-token distribution
\begin{equation}
p_\theta(t_j \mid \mathcal{I}, \mathcal{T}_{<j}),
\label{eq:next_token_dist}
\end{equation}
where $\mathcal{T}_{<j}=\mathcal{T}_v \oplus \mathcal{T}_p \oplus \mathcal{T}_a^{<j}$ is the full prefix context.
To localize $t_j$, we extract a token-to-vision activation map
$A_{j,\cdot}\in\mathbb{R}^{N_v}$ (definition and aggregation are given in \cref{subsec:problem}).
Ideally, the spatial evidence supporting token $t_j$ should reflect the image $\mathcal{I}$ and the semantics of $t_j$,
and thus be (approximately) invariant to irrelevant variations in the prefix context. Concretely, we desire
\begin{equation}
A_{j,\cdot} \approx g_\theta(\mathcal{I}, t_j),
\label{eq:ideal_dep}
\end{equation}
while in practice, because the model is causal and autoregressive,
\begin{equation}
A_{j,\cdot} = g_\theta(\mathcal{I}, \mathcal{T}_{\le j}),
\qquad \mathcal{T}_{\le j}=\mathcal{T}_v \oplus \mathcal{T}_p \oplus \mathcal{T}_a^{<j} \oplus t_j,
\label{eq:actual_dep}
\end{equation}
so $A_{j,\cdot}$ is generally \emph{prefix-conditioned}.

We empirically attribute the mismatch between \cref{eq:ideal_dep} and \cref{eq:actual_dep}
to three confounding pathways:
(i) \textbf{prompt-induced priors} from $\mathcal{T}_p$,
(ii) \textbf{generated-answer priors} from $\mathcal{T}_a^{<j}$ through residual accumulation, and
(iii) \textbf{language-model priors} dominating when visual evidence is ambiguous. Therefore, 
the next-token distribution \cref{eq:next_token_dist} becomes less sensitive to $\mathcal{I}$ and more driven by textual context.

\begin{proposition}[Autoregressive Context Contamination]
\label{prop:context}
Consider a causal self-attention Transformer in which, for at least one head used by the attribution extractor,
the attention from position $j$ to some previous position $j'<j$ can be non-zero.
Then the representation at step $j$ and any activation map derived from it depend on the previous context. Let $\widetilde{\mathcal{T}}_{<j}$ denote an alternative prefix context obtained by changing at least one token in $\mathcal{T}_{<j}$ while keeping the image $\mathcal{I}$ fixed, where $\widetilde{A}_{j,\cdot}$ is the activation map produced under $\widetilde{\mathcal{T}}_{<j}$.
In particular, there exist prefixes $\mathcal{T}_{<j}$ and $\widetilde{\mathcal{T}}_{<j}$ such that
\begin{equation}
(\mathcal{I}, t_j)\ \text{fixed},\ \mathcal{T}_{<j}\neq \widetilde{\mathcal{T}}_{<j}
\quad\Rightarrow\quad
A_{j,\cdot} \neq \widetilde{A}_{j,\cdot}.
\label{eq:context_ineq}
\end{equation}
\end{proposition}
Note: Full proof is provided in the supplementary material (Sec.~A.1).

\begin{remark}[When can contamination disappear?]
Equality in \cref{eq:context_ineq} would require that the attribution-relevant representation at position $j$
is effectively insensitive to all earlier tokens,
which is a highly restrictive condition and does not hold in standard autoregressive Transformers.
\end{remark}

\begin{remark}[Connection to PV-TAM]
Proposition~\ref{prop:context} motivates extracting attribution using \emph{prompt-side} query tokens.
Specifically, PV-TAM places target concepts as fixed prompt tokens and computes their token-to-vision activation maps
from representations that do not depend on the \emph{variable} generated prefix $\mathcal{T}_a^{<j}$.
This avoids the autoregressive contamination in \cref{eq:context_ineq} and yields more stable token--region alignment.
\end{remark}

\begin{figure}[htbp]
    \centering    \includegraphics[width=1.\linewidth]{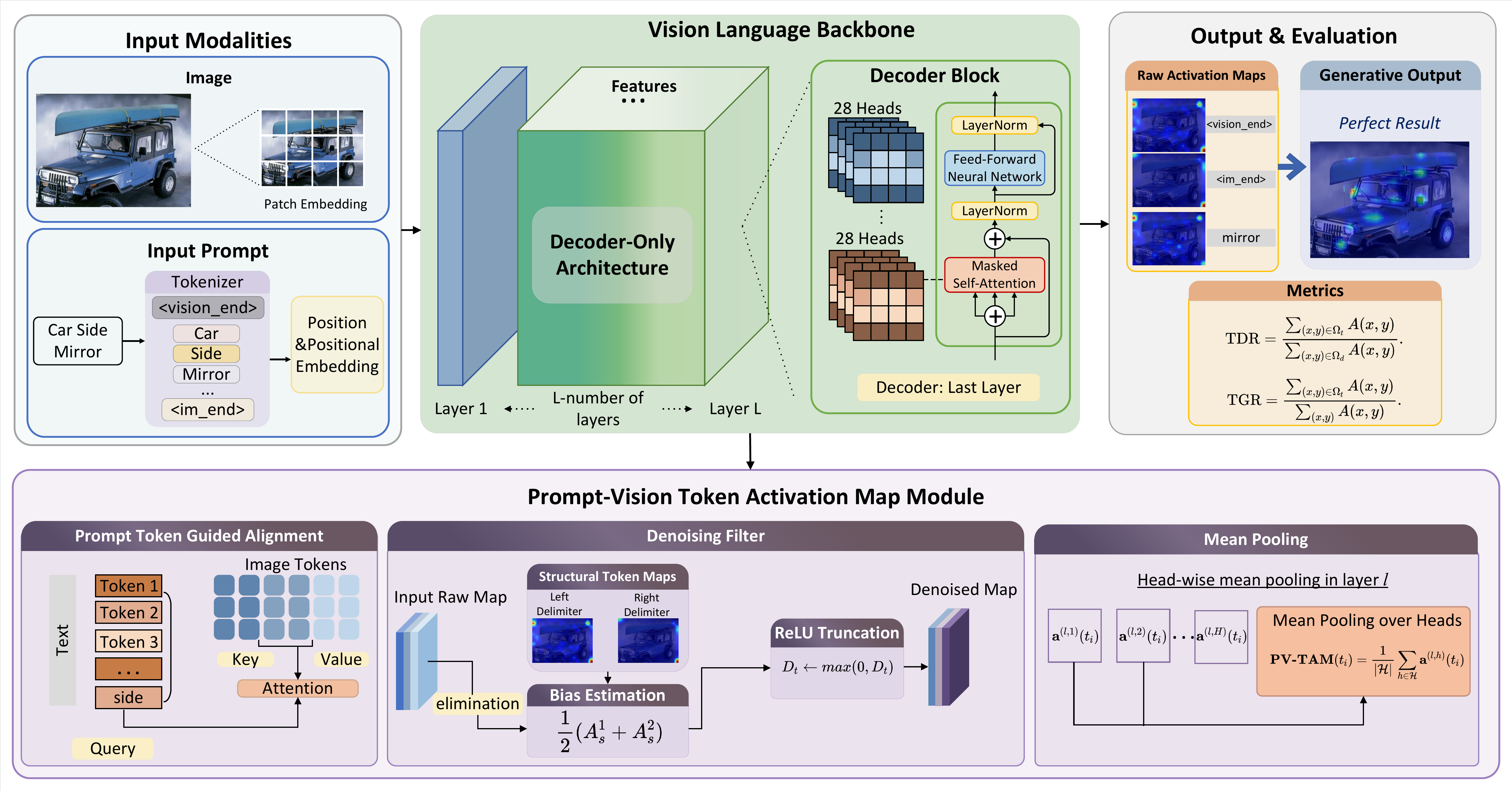}
    \caption{PV-TAM framework. It builds on a decoder-only backbone to perform alignment. The raw attention map is taken from the last-layer attention weights. We apply three components in refining attention map:(i) prompt-token guided alignment, using the prompt token as the query to attend to vision tokens and reduce semantic interference;(ii) a denoising filter to remove systematic bias introduced by structural tokens;(iii) head-wise mean pooling to aggregate attention maps.} 
    \label{fig:pvtam}
\end{figure}

\subsection{PV-TAM: Prompt-Vision Token Activation Map}
\subsubsection{PV-TAM Extraction.}
To reduce inconsistency, we use \textbf{prompt-side target tokens} rather than answer-side tokens. As shown in \Cref{fig:pvtam},
given a set of target tokens $\mathcal{T}_{\text{target}}=\{t_1,\ldots,t_m\}$, we construct an augmented sequence
\begin{equation}
\mathcal{T}_{\text{aug}}=\mathcal{T}_v \oplus \mathcal{T}_{\text{target}} \in \mathbb{R}^{(N_v+m)\times d},
\end{equation}

For each target token $t_i$, we extract its prompt-to-vision attention vector $\mathbf{a}^{(l,h)}(t_i)$ and aggregate across heads by mean pooling:
\begin{equation}
\mathbf{PV\text{-}TAM}^{(l)}(t_i)=\frac{1}{|\mathcal{H}|}\sum_{h\in\mathcal{H}}\mathbf{a}^{(l,h)}(t_i)\in\mathbb{R}^{N_v}.
\end{equation}

In general, we use the attention from the last transformer layer. 
The resulting vector is reshaped to $H'\times W'$ and upsampled to the image resolution.
\subsubsection{Structural Bias Elimination.}
Prompt-side attention can still contain spatial bias introduced by special structure tokens, like modality delimiters(e.g., \texttt{<vision\_start>}, \texttt{<vision\_end>}, \texttt{<im\_start>}, \texttt{<im\_end>}).
Let $\bm{A}^t$ denote the PV-TAM of a target token, then $\bm{A}_1^s,\bm{A}_2^s$ denote the PV-TAM maps of the two adjacent special tokens.
We estimate the structure bias by their mean and subtract it from the target response:
\begin{equation}
\bm{D}_t=\bm{A}^{t}-\frac{1}{2}\left(\bm{A}^{s}_{1}+\bm{A}^{s}_{2}\right),
\label{eq:decontam}
\end{equation}
followed by ReLU truncation to keep positive evidence:
\begin{equation}
\bm{D}_t \leftarrow \max(\bm{0},\,\bm{D}_t).
\label{eq:decontam_relu}
\end{equation}

\begin{assumption}[Additive decomposition]
\label{assump:additive}
For each visual position $i \in [N_v]$, the target token's attention
decomposes as
\begin{equation}
  A^t_i = A^{\mathrm{sig}}_i + A^{\mathrm{sp}}_i + \varepsilon^t_i,
\end{equation}
where $\bm{A}^{\mathrm{sig}} \geq \bm{0}$ is the semantic information,
$\bm{A}^{\mathrm{sp}} \geq \bm{0}$ is the structural bias shared with
neighboring special tokens, and $\bm{\varepsilon}^t$ is residual noise
with $\mathbb{E}[\varepsilon^t_i] = 0$.
\end{assumption}

\begin{assumption}[Weak semantic decomposition]
\label{assump:weak}
Special tokens carry negligible semantic
information, so their attention maps satisfy
\begin{equation}
  A^{s_k}_i = A^{\mathrm{sp}}_i + \delta^{(k)} , k \in \{1, 2\},
\end{equation}
where $\delta^{(k)}_i$ is a position-dependent deviation.
\end{assumption}


\begin{assumption}[Unbiased boundary deviations]
\label{assump:unbiased}
We assume the deviations
have zero mean under the randomness of previous contexts:
\begin{equation}
  \mathbb{E}\!\left[\delta^{(k)}_i \mid \mathcal{I},\mathcal{T}_v,\mathcal{T}_p\right] = 0,
  \quad \forall\, i\in[N_v],\; k\in\{1,2\}.
\end{equation}
\end{assumption}

\begin{proposition}[Structural Bias Elimination]
\label{prop:decontam}
Under Assumptions~\ref{assump:additive}--\ref{assump:unbiased}, 
the unbiased map satisfies
\begin{equation}
  \bm{D}_t
  \;:=\;
  \bm{A}^t - \tfrac{1}{2}\bigl(\bm{A}_1^s + \bm{A}_2^s\bigr)
  \;=\;
  \bm{A}^{\mathrm{sig}} + \bm{\varepsilon}^t,
\end{equation}
where $\bm{A}^{\mathrm{sig}}$ is the target's semantic signal and
$\bm{\varepsilon}^t$ is zero-mean residual noise.
\end{proposition}
Full proof is provided in the supplementary material (Sec.~A.2).

\begin{corollary}[$\ell_1$ error bound after ReLU truncation]
\label{cor:relu_bound}
Let $\hat{\bm{D}}_t = \max(\bm{0},\, \bm{D}_t)$ denote the ReLU-truncated
map ( \cref{eq:decontam_relu}).
Then
\begin{equation}
  \bigl\|\hat{\bm{D}}_t - \bm{A}^{\mathrm{sig}}\bigr\|_1
  \;\leq\;
  \|\bm{\varepsilon}^t\|_1,
\end{equation}
with equality if \/ $\bm{\varepsilon}^t = \bm{0}$.
\end{corollary}

\begin{proof}
Since $\bm{A}^{\mathrm{sig}} \geq \bm{0}$ and ReLU is $1$-Lipschitz,
\begin{equation}
  \bigl\|\hat{\bm{D}}_t - \bm{A}^{\mathrm{sig}}\bigr\|_1
  = \bigl\|\max(\bm{0},\,\bm{A}^{\mathrm{sig}} + \bm{\varepsilon}^t)
           - \bm{A}^{\mathrm{sig}}\bigr\|_1
  \leq \|\bm{\varepsilon}^t\|_1. \qed
\end{equation}
\end{proof}

Therefore, using only one flanking special token introduces a systematic error of $\delta^{(k)}_i$ per position, while bilateral mean cancels this error.
\subsubsection{Foreground-Guided Refinement.}
\label{subsubsec:fore}
Even after denoising, salient but irrelevant background regions may dominate the response.
We optionally apply a foreground probability mask $\bm{F}\in[0,1]^{H\times W}$ obtained from an off-the-shelf background removal model (rembg) and refine the activation by
\begin{equation}
\bm{A}^{\text{refined}}_t = \bm{D}_t \odot \bm{F},
\end{equation}
where $\odot$ is the element-wise product.
We report results both with and without this refinement to separate the core attribution signal from external mask priors.

\subsection{Evaluation for Prompt-Token Alignment}
\label{sec:metrics}
To quantitatively assess the alignment between an activation map and the target semantic part, we introduce three complementary metrics to capture the reliability of token-to-region alignment: (i) the fraction of attention weights concentrated inside the target region, (ii) whether the dominant attention corresponds to the target region, and (iii) the geometric distance between salient peaks and the target center. In contrast to IoU-style metrics that depend solely on overlap between binarized masks, our metrics directly measure the \emph{attention distribution} and \emph{salient responses} relative to the target, making them more sensitive to failure modes such as attention drift, background-induced saliency, and multi-peak activations.

Let $A \in \mathbb{R}^{H\times W}$ denote the PV-TAM of target token, and $\Omega_t \subseteq \{1,\dots,H\} \times \{1,\dots,W\}$ denote ground-truth mask. For PartImageNet \cite{he2022partimagenet}, we obtain the binary activation mask $\Omega_d$ of the foreground object using rembg~\cite{gatis2022rembg}, a well-maintained background removal tool based on U$^2$-Net~\cite{qin2020u2net}. 
We discard samples where rembg fails to cover the ground-truth mask, ensuring $\Omega_d \cap \Omega_t = \Omega_t$. Since RefCOCO\cite{refcoco} requires global information for semantic understanding, we do not apply foreground-guided refinement in \cref{subsubsec:fore}.
\subsubsection{\textbf{Target--Global Ratio (TGR)}}
TGR measures the proportion of total activation energy that falls inside the target-part region.
A higher TGR indicates that the explanation is globally concentrated on the intended target.
\begin{equation}
\mathrm{TGR} =
\frac{\sum_{(x,y)\in\Omega_t} A(x,y)}
     {\sum_{(x,y)} A(x,y)}.
\end{equation}
\subsubsection{\textbf{Target--Dominant Ratio (TDR)}}



Let $\Omega_d$ denote the dominant region extracted by rembg~\cite{gatis2022rembg}. It is defined as:
\begin{equation}
\mathrm{TDR} = \frac{\sum_{(x,y)\in\Omega_t} A(x,y)}{\sum_{(x,y)\in\Omega_d} A(x,y)}.
\end{equation}
By restricting the evaluation to the foreground mask, attention can more effectively concentrate on target-related regions within the object, free from interference by background noise.


\subsubsection{\textbf{Min Distance (Min-Dist)}}

We extract a set of candidate saliency peaks $\{p_{k}\}_{k=1}^{K}$ from one PV-TAM. Let $\mathcal{M}$ denote the target mask representing the ground-truth part region. To eliminate resolution mismatch, we compute distances in normalized coordinates $[0,1]\times[0,1]$. The distance from peak $p_k$ to the mask $\mathcal{M}$ is defined as the minimum distance to all pixels in the mask:
\begin{equation}
d_{k}=\min_{m \in \mathcal{M}} \sqrt{\frac{(p_{k}[x]-m[x])^2}{W^2}+\frac{(p_{k}[y]-m[y])^2}{H^2}},
\end{equation}
where $W,H$ are the image width and height. For sample $s$, the minimum distance is obtained as:
\begin{equation}
\hat{d}_s=\min_{k} d_{k}.
\end{equation}
In our experiments, the maximum value of $K$ is set to 5. Nevertheless, scenarios where $K < 5$ are encountered in practice. In the extreme case where $K=0$ (i.e., no peaks detected), we assign $\hat{d}_s=1$ as the worst-case value. The dataset-level MinDist is formulated as:
\begin{equation}
\mathrm{MinDist}=\frac{1}{N}\sum_{s=1}^{N}\hat{d}_s.
\end{equation}
Lower MinDist indicates better alignment between salient peak and the true center.
\subsubsection{\textbf{Obj-IoU and Mean-IoU.}}
Following the format of Object-IoU in TAM\cite{li2025tam}, we compute the intersection-over-union (IoU) between the binary activation map and the ground truth object mask. Unlike the Otsu \cite{otsu1975threshold} thresholding method used in TAM, we apply spatial mean thresholding to each activation map and use the resulting IoU value as \textbf{Mean-IoU}.

\section{Experiments}
\subsection{Setup}
We use multiple open-source VLMs, including Qwen2-VL\cite{wang2024qwen2vl}, Qwen2.5-VL\cite{bai2025qwen25vl} and InternVL\cite{chen2024internvl}. 
For answer-side baselines, we standardize the prompt as
\texttt{"Focus on the target in the image. Note that only the word `target' can be returned."}
For baselines that require binarization, we follow the original implementation and apply Otsu\cite{otsu1975threshold} thresholding to obtain binary masks.
\subsection{Datasets}

\textbf{PartImageNet}\cite{he2022partimagenet} is a part-segmentation benchmark for fine-grained region alignment and a referring-expression benchmark for language-conditioned localization.
 It provides pixel-level part annotations for object categories, enabling quantitative evaluation of whether a token-to-visual activation map aligns with the target part region. 
\textbf{RefCOCO}\cite{refcoco} is a referring expression comprehension dataset where each sample consists of an image, a natural language referring expression, and a ground-truth bounding box for the referred object. We utilize RefCOCO to investigate whether object semantics incorporating adjectival modifiers can be aligned.

\subsection{Quantitative analysis}
We report a quantitative evaluation of prompt-to-region consistency across multiple models and datasets. To encourage stronger consistency in answer-side attention, we compare our approach with state-of-the-art interpretability methods, including TAM\cite{li2025tam}, Grad-CAM\cite{selvaraju2017grad}, CP-LRP\cite{ali2022xai}, and Attention Rollout\cite{abnar2020quantifying}. 

Overall, Table~\ref{tab:main} presents results under Qwen2.5-VL-7B\cite{bai2025qwen25vl}. Our method achieves the best TGR/TDR and the lowest MinDist, indicating that attention is both more concentrated on the target region and more accurately localized in terms of peak position.
\begin{table}[tb]
\centering
\caption{Comparison of part-aware explanation quality on PartImageNet-OOD using Qwen2-VL-7B model\cite{wang2024qwen2vl}. $\uparrow$ indicates higher is better. $^{\#}$ denotes results with emphasis on the target token. We use the attention from the last transformer layer, but for Attention Rollout\cite{ali2022xai}, we use the last three layers. }
\label{tab:main}
\scriptsize
\begin{tabularx}{\linewidth}{Xcccc}
\toprule
Method & TGR$\uparrow( \times 10^{2})$ & TDR$\uparrow( \times 10^{2})$  & Min Dist$\downarrow(\times 10^{-1})$ \\
\midrule
CAM~\cite{zhou2016learning}                   & 7.03 & 5.25 & 0.77 \\
CAM$^{\#}$ ~\cite{zhou2016learning}          & 8.47 & 6.2  & 0.72 \\
CP-LRP~\cite{ali2022xai}               & 4.12  & 6.02 & 0.71 
\\
CP-LRP$^{\#}$ \cite{ali2022xai}        & 3.63 & 6.02 & 0.71 
\\
Attention Rollout~\cite{abnar2020quantifying} & 2.46 & 3.09 & 1.62 \\
TAM\cite{li2025tam}                           & 6.63 & 5.21 & 1.25 \\
TAM$^{\#}$ \cite{li2025tam}                  & 8.29 & 6.36 & 1.10 \\
\midrule
Ours(Prompt-side) & \textbf{9.67}  & \textbf{7.26}  & \textbf{0.65}  \\
\bottomrule
\end{tabularx}
\end{table}
\Cref{tab:qwens} reports the cross-model results covering both the Qwen series and InternVL models. Our method achieves strong performance on medium-scale models such as Qwen2-VL-7B and Qwen2.5-VL-7B, with the best results obtained by Qwen2.5-VL-7B. Notably, the proposed method also demonstrates clear advantages on smaller models, outperforming baseline on InternVL3-1B.
\begin{table}[tb]
\centering
\caption{Cross-model generalization under Qwen and InternVL in PartImageNet-OOD.}
\label{tab:qwens}
\scriptsize
\setlength{\tabcolsep}{0pt}
\begin{tabular*}{\linewidth}{@{\extracolsep{\fill}}lcccccc}
\toprule
\multirow{2}{*}{Backbone} & \multicolumn{2}{c}{TGR$\uparrow( \times 10^{2})$} & \multicolumn{2}{c}{TDR$\uparrow( \times 10^{2})$} & \multicolumn{2}{c}{Min Dist$\downarrow(\times 10^{-1})$} \\
\cmidrule(lr){2-3} \cmidrule(lr){4-5} \cmidrule(lr){6-7}
& Ours & TAM$^{\#}$ & Ours & TAM$^{\#}$ & Ours & TAM$^{\#}$ \\
\midrule
Qwen2-VL-2B\cite{wang2024qwen2vl}    & $4.87$ & $\textbf{8.13}$ & $4.08 $ & $\textbf{5.21} $ & $\textbf{0.91} $ & $1.24 $ \\
Qwen2-VL-7B\cite{wang2024qwen2vl}   & $\textbf{9.67}$ & $8.29$ & $\textbf{7.26}$ & $6.36$ & $\textbf{0.65}$ & $1.10 $ \\
Qwen2.5-VL-3B\cite{bai2025qwen25vl} & $4.21$ & $\textbf{5.02}$ & $2.96 $ & $\textbf{3.23} $ & $\textbf{1.12}$ & $1.26 $ \\
Qwen2.5-VL-7B\cite{bai2025qwen25vl} & $\textbf{7.60}$ & $5.34$ & $\textbf{6.08}$ & $5.26 $ & $\textbf{0.92}$ & $1.04 $ \\
InternVL3-1B\cite{zhu2025internvl3exploringadvancedtraining} & $\textbf{3.42}$ & $2.36$ & $\textbf{5.70}$ & $3.34$ & $\textbf{0.81}$ & $1.47$ \\
\bottomrule
\end{tabular*}
\end{table}

To provide a more fine-grained analysis across different semantic categories, \Cref{tab:perclass} reports the Obj-IoU and win-rate comparisons. Overall, our method achieves performance comparable to the SOTA answer-side interpretability approach TAM\cite{li2025tam} (Baseline). Specifically, on Qwen2-VL-7B\cite{wang2024qwen2vl}, our method shows stronger attention to fine-grained semantic parts such as bicycle head, aeroplane head, aeroplane tail, and bottle mouth, indicating improved localization of semantically relevant regions under distribution shift.

\begin{table}[htbp]
\centering
\caption{Class-level comparison on PartImageNet-OOD. We report the mean Object-level IoU (Obj-IoU) used in TAM\cite{li2025tam}. Win rate is percentage of samples where our method outperforms baseline.}
\label{tab:perclass}
\scriptsize
\setlength{\tabcolsep}{0pt}
\begin{tabular*}{\linewidth}{@{\extracolsep{\fill}}lcccccc}
\toprule
\multirow{2}{*}{Class} & \multicolumn{3}{c}{Qwen2-VL-7B\cite{wang2024qwen2vl}} & \multicolumn{3}{c}{Qwen2.5-VL-7B\cite{bai2025qwen25vl}} \\
\cmidrule(lr){2-4} \cmidrule(lr){5-7}
 & Baseline& Ours& Win-rate$\uparrow$ & Baseline & Ours& Win-rate$\uparrow$ \\
\midrule
Car side mirror   & \textbf{2.02 } & 1.99  & \textbf{57.62} & 1.98  & \textbf{5.05}  & \textbf{68.04}  \\
Bicycle head      & 2.14  & \textbf{3.11}  & \textbf{81.58} & \textbf{3.59}  & 1.13  & \textbf{51.79}  \\
Bicycle seat      & \textbf{5.58}  & 4.99  & \textbf{58.08} & \textbf{6.40}  & 1.42  & 27.94  \\
Boat sail         & \textbf{25.37} & 17.72 & 34.19 & \textbf{24.86} & 14.01 & 23.22  \\
Aeroplane head    & 18.97 & \textbf{26.56} & \textbf{68.66} & \textbf{19.00} & 13.46 & 41.26   \\
Aeroplane engine  & \textbf{12.88} & 11.50 & 44.20 & \textbf{11.47} & 5.93  & 26.15  \\
Aeroplane wing    & \textbf{10.36} & 9.00  & 44.92 & \textbf{9.51}  & 8.14  & \textbf{50.87}  \\
Aeroplane tail    & 14.81 & \textbf{20.35} & \textbf{64.00} & 8.00  & \textbf{13.79} & \textbf{71.64}  \\
Bottle mouth      & 4.01  & \textbf{15.55} & \textbf{87.72} & \textbf{13.39} & 10.55 & 48.28  \\
Bottle body       & 18.37 & \textbf{25.32} & \textbf{61.90} & \textbf{26.15} & 22.58 & 41.18  \\
\midrule
Total             & 10.13 & \textbf{10.26} &\textbf{ 57.20} & \textbf{10.53} & 7.86  & 46.81 \\
\bottomrule
\end{tabular*}
\end{table}

For the prompt-side approach proposed in this paper, we systematically compared prompt-side and answer-side methods in \Cref{tab:token_source1} and \Cref{tab:token_source}. Experimental results demonstrate that on the Qwen-VL-7B\cite{wang2024qwen2vl}, the prompt-side method exhibits significant advantages, with overall performance superior to various answer-side methods. On the PartImageNet\cite{he2022partimagenet} and RefCOCO\cite{refcoco} datasets, this method achieves optimal results on both attention intensity metrics and IoU-style metrics, indicating that the method not only accurately focuses on target semantic regions but also exhibits higher attention. These results further validate that autoregressive answer tokens are prone to contextual drift, whereas prompt-side strategies provide more stable visual-language alignment.


\begin{table}[tb]
\centering
\caption{\textbf{Controlled study on token sources.} Most baselines define token-to-visual maps on \emph{answer-side} tokens, thus prompt-side maps are not applicable. Our method supports both prompt-side PV-TAM and answer-side extraction under identical evaluation on datasets \textbf{PartImageNet}. Prompt-side results are not applicable for baselines due to definition mismatch. \textsuperscript{$\dagger$}TGR, TDR are reported in units of $10^2$, and MinDist in $10^{-1}$.$^{\#}$ represents last token.}
\label{tab:token_source1}
\scriptsize
\setlength{\tabcolsep}{0pt}
\begin{tabular*}{\linewidth}{@{\extracolsep{\fill}}lllccccc}
\toprule
Backbone & Method & Source & TGR$\uparrow$ & TDR$\uparrow$ & Obj-IOU$\uparrow$ & MeanIOU$\uparrow$ & MinDist$\downarrow$ \\
\midrule
\multirow{9}{*}{Qwen2-VL-2B\cite{wang2024qwen2vl}}
& \textbf{Ours}     & prompt        & 4.87 & 4.08 & 8.23 & 8.48 & \textbf{0.91} \\
& Ours              & answer        & 4.19 & 4.76 & 5.46 & 7.97 & 1.42 \\
& TAM\cite{li2025tam}               & answer        & 5.50 & 3.54 & 7.50 & 8.18 & 1.63 \\
& TAM\cite{li2025tam}                 & answer$^{\#}$ & \textbf{8.13} & \textbf{5.21} & \textbf{10.11}& \textbf{9.29} & 1.24 \\
& CAM\cite{jiang2021layercam}                & answer        & 4.28 & 3.47 & 7.37 & 8.05 & 0.98 \\
& CAM\cite{jiang2021layercam}               & answer$^{\#}$ & 5.44 & 4.00 & 7.77 & 8.15 & 0.95 \\
& Attention Rollout\cite{abnar2020quantifying} & answer        & 1.46 & 2.47 & 6.08 & 7.79 & 1.02 \\
& CP-LRP\cite{ali2022xai}            & answer        & 1.50 & 3.16 & 4.08 & 7.26 & 1.20 \\
& CP-LRP\cite{ali2022xai}               & answer$^{\#}$ & 1.61 & 3.19 & 3.56 & 6.77 & 1.14 \\
\midrule
\multirow{9}{*}{Qwen2-VL-7B\cite{wang2024qwen2vl}}
& \textbf{Ours}     & prompt        & \textbf{9.67} & \textbf{7.26} & \textbf{10.26}& \textbf{9.82} &\textbf{ 0.65} \\
& Ours              & answer        & 4.96 & 4.77 & 7.41 & 8.56 & 0.97 \\
& TAM\cite{li2025tam}                 & answer        & 6.63 & 5.21 & 7.11 & 7.25 & 1.25 \\
& TAM\cite{li2025tam}                 & answer$^{\#}$ & 8.29 & 6.36 & 10.13& 8.95 & 1.10 \\
& CAM\cite{jiang2021layercam}                & answer        & 7.03 & 5.25 & 7.80 & 8.12 & 0.77 \\
& CAM\cite{jiang2021layercam}                & answer$^{\#}$ & 8.47 & 6.20 & 8.61 & 8.38 & 0.72 \\
& Attention Rollout\cite{abnar2020quantifying} & answer        & 2.46 & 3.09 & 8.02 & 8.52 & 1.62 \\
& CP-LRP\cite{ali2022xai}               & answer        & 4.12 & 4.61 & 6.60 & 8.16 & 0.71 \\
& CP-LRP\cite{ali2022xai}               & answer$^{\#}$ & 3.63 & 4.39 & 6.49 & 8.14 & 0.71 \\
\midrule
\multirow{9}{*}{Qwen2.5-VL-7B\cite{bai2025qwen25vl}}
& \textbf{Ours  }   & prompt        & 7.60 &\textbf{ 6.08} & 7.86 & 8.36 & 0.92 \\
& Ours              & answer        & \textbf{8.42} & 5.71 & 5.23 & 6.97 & 1.18 \\
& TAM\cite{li2025tam}                 & answer        & 5.78 & 4.06 & 7.49 & 7.71 & 1.22 \\
& TAM\cite{li2025tam}                 & answer$^{\#}$ & 5.34 & 5.26 & \textbf{10.51}&\textbf{ 9.17} & 1.04 \\
& CAM\cite{jiang2021layercam}                & answer        & 3.97 & 5.06 & 8.17 & 8.17 & 1.10 \\
& CAM\cite{jiang2021layercam}                & answer$^{\#}$ & 4.67 & 4.35 & 8.84 & 8.33 & 0.89 \\
& Attention Rollout\cite{abnar2020quantifying} & answer        & 2.80 & 3.03 & 8.11 & 8.26 & 1.61 \\
& CP-LRP\cite{ali2022xai}               & answer        & 5.11 & 4.42 & 6.55 & 7.69 & 0.87 \\
& CP-LRP\cite{ali2022xai}               & answer$^{\#}$ & 5.81 & 4.51 & 6.50 & 7.72 & \textbf{0.81 }\\
\bottomrule
\end{tabular*}
\\[6pt]
\footnotesize

\end{table}

\begin{table}[tb]
\centering
\caption{\textbf{Controlled study on token sources.} Most baselines define token-to-visual maps on \emph{answer-side} tokens, thus prompt-side maps are not applicable. Our method supports both prompt-side PV-TAM and answer-side extraction under identical evaluation on datasets \textbf{RefCOCO}. \textsuperscript{$\dagger$}TGR and MinDist are reported in units of $10^2$ and $10^{-1}$, respectively.}
\label{tab:token_source}
\scriptsize
\setlength{\tabcolsep}{0pt}
\begin{tabular*}{\linewidth}{@{\extracolsep{\fill}}lllcccc}
\toprule
Backbone & Method & Source &  TGR$\uparrow$ & Obj-IOU$\uparrow$ & Mean-IOU$\uparrow$ & MinDist$\downarrow$ \\
\midrule
\multirow{7}{*}{Qwen2-7B}
& \textbf{Ours}     & \textbf{prompt}        & \textbf{15.77}& \textbf{14.27}&\textbf{22.42}&\textbf{ 0.23}\\
& Ours              & answer        & 4.39 & 2.92 &10.51& 1.97\\
& CAM\cite{jiang2021layercam}                & answer        & 4.64 & 9.05 &10.85& 0.56\\
& CAM\cite{jiang2021layercam}                & answer$^{\#}$ & 5.23 & 11.58&12.58& 0.55\\
& Attention Rollout\cite{abnar2020quantifying} & answer        & 5.96 & 5.54 &13.42& 0.77\\
& CP-LRP\cite{ali2022xai}               & answer        & 5.22 & 6.57 &11.57& 0.64\\
& CP-LRP\cite{ali2022xai}               & answer$^{\#}$ & 6.07 & 8.43 &12.39& 0.41\\
\midrule
\multirow{7}{*}{Qwen2.5-7B}
& \textbf{Ours }    & \textbf{prompt}        & \textbf{8.98} & 9.47 &\textbf{16.78}&\textbf{ 0.46}\\
& Ours              & answer        & 5.84 & 4.41 & 9.06&1.09\\
& CAM\cite{jiang2021layercam}                & answer        & 3.10 & 8.88 & 9.28& 1.10\\
& CAM\cite{jiang2021layercam}                & answer$^{\#}$ & 3.48 & \textbf{10.20}&10.45& 0.93\\
& Attention Rollout\cite{abnar2020quantifying} & answer        & 6.15 & 7.43 &12.70& 0.53\\
& CP-LRP\cite{ali2022xai}               & answer        & 4.62 & 5.49 &10.67& 0.82\\
& CP-LRP\cite{ali2022xai}               & answer$^{\#}$ & 4.91 & 6.38 &11.22& 0.70\\
\bottomrule
\end{tabular*}
\end{table}

In the ablation study (\Cref{tab:ablation}), we analyze the contribution of each module. 
The results show that both the prompt-side semantic module and the denoising module are indispensable for achieving attention-aligned PVTAM and obtaining strong performance.
\begin{table}[tb]
\centering
\scriptsize
\caption{\textbf{Ablation study} on PartImageNet-OOD and RefCOCO using Qwen-2.5-VL-7B. We isolate each component's contribution by removing or replacing one factor at a time. \texttt{w/o Filter} denotes removing the denoising module. \texttt{w/o Prompt-side} denotes using the generated answer-side as text input.}
\label{tab:ablation}
\setlength{\tabcolsep}{0pt}
\begin{tabular*}{\linewidth}{@{\extracolsep{\fill}}lccccccc}
\toprule
 & \multicolumn{3}{c}{\textbf{PartImageNet\cite{he2022partimagenet}}} & \multicolumn{4}{c}{\textbf{RefCOCO\cite{refcoco}}} \\
\cmidrule(lr){2-4} \cmidrule(lr){5-8}
 & \textbf{TGR}$\uparrow$ & \textbf{TDR}$\uparrow$ & \textbf{MinDist}$\downarrow$ & \textbf{TGR}$\uparrow$ & \textbf{Obj-IOU}$\uparrow$ & \textbf{Mean-IOU}$\uparrow$ &\textbf{MinDist}$\downarrow$ \\
\midrule
All Modules & 7.60 & \textbf{6.08} & \textbf{0.92} & \textbf{8.98} & \textbf{9.47} & \textbf{16.78} & \textbf{0.46} \\
w/o Filter & 4.27 & 2.58 & 1.18 & 3.30 & 1.42 & 10.57 & 1.85 \\
w/o Prompt-side & \textbf{8.42} & 5.71 & 1.18 & 5.84 & 4.41& 9.06  & 1.09 \\
\bottomrule
\end{tabular*}
\end{table}

\subsection{Qualitative analysis}

To provide a more intuitive analysis of PV-TAM's attention behavior, \Cref{fig:examples} presents visualization results from PartImageNet\cite{he2022partimagenet} and RefCOCO\cite{refcoco}. The first row of examples, sourced from PartImageNet\cite{he2022partimagenet}, analyzes the model's ability to focus on detailed regions of objects. Comparing against the ground-truth mask reveals that PV-TAM achieves higher localization accuracy in detail regions, capturing object structural boundaries more precisely.

The subsequent three examples are from the RefCOCO\cite{refcoco}. In complex scenes containing multiple objects, PV-TAM accurately localizes the target subject, whereas TAM exhibits noticeable attention diffusion. While TAM partially focuses on the target entity, inaccuracies persist in boundary regions. In contrast, CAM's attention distribution is relatively dispersed, and other methods largely fail to effectively focus on the target area.

The final two \textbf{Dish} examples demonstrate attention performance in the presence of similar distracting objects. Results indicate that TAM\cite{li2025tam} struggles to distinguish similar targets based on attributive information, whereas PV-TAM accurately identifies the target object and maintains clear boundary localization.

\begin{figure}[tb]
\centering
\includegraphics[width=1.\linewidth]{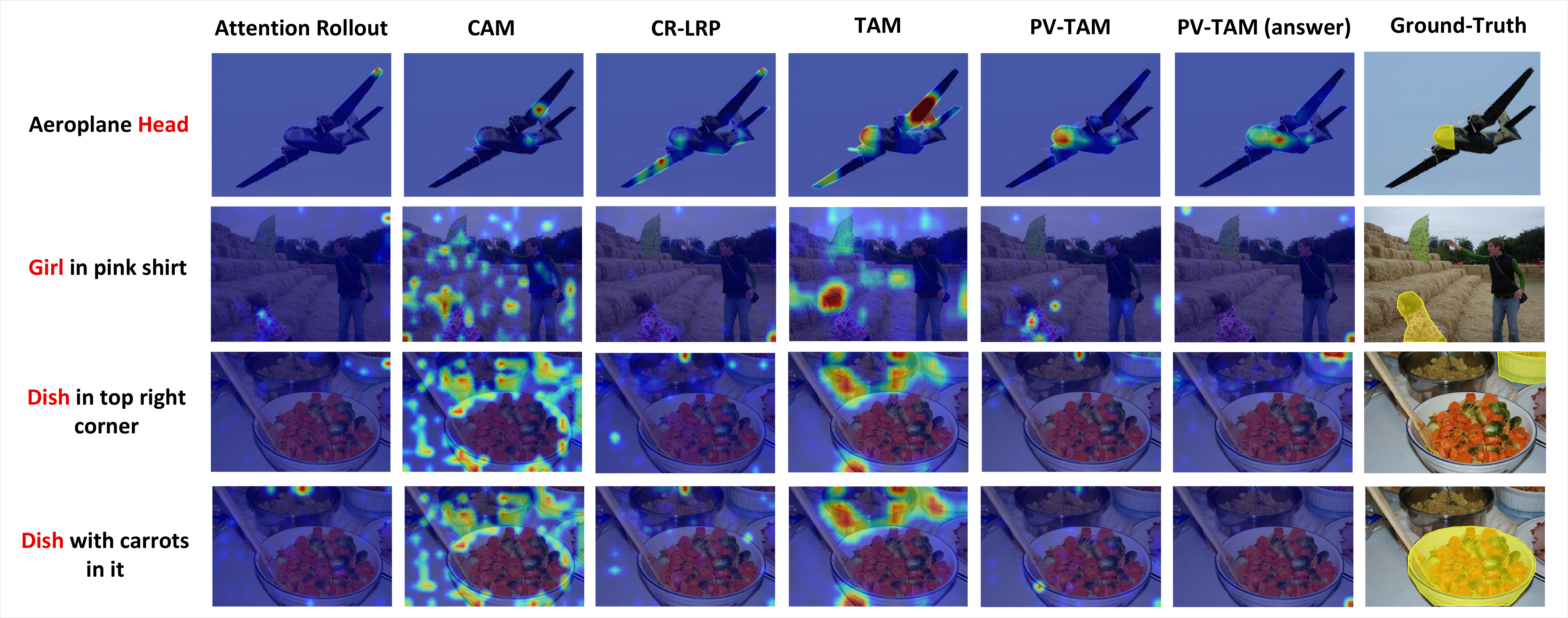}
    \caption{Visual comparison between PV-TAM and baselines. The red highlights the target token. The expected high level of attention indicated by these red areas should reside in the semantic region corresponding to that token; the more precise the better.} 
    \label{fig:examples}
\end{figure}
\section{Discussion}
It is noteworthy that the technical reports for Qwen2.5-VL\cite{bai2025qwen25vl} and InternVL\cite{chen2024internvl} did not explicitly treat token-level alignment as a training objective. However, the results in \Cref{tab:token_source} indicate that \emph{prompt-side} token extraction consistently generates stronger local attention compared to \emph{answer-side}. This phenomenon suggests that fine-grained word-region alignment may naturally emerge as a byproduct during the training of VLMs.

We hypothesize that this arises from the functional constraints imposed by the prompt-following generation process: to produce correct, visually grounded outputs, the model must associate component-relevant semantic units in the prompt with local visual information. This establishes stable attention patterns even without token-level alignment losses during training.

\textbf{Implications for Knowledge Distillation, Training.}
Prompt-side explanations can serve as lightweight grounding signals for knowledge distillation: activation maps function as soft pseudo-masks, providing component labels and visual region alignment supervision for smaller models. Additionally, the proposed metrics can be leveraged to introduce alignment-aware regularization constraints during training, further enhancing visual-language alignment capabilities.

\textbf{Limitations.}
Attention alignment does not necessarily imply causal explanation, and the gains are category-dependent (Tables~\ref{tab:perclass}). Developing class-aware prompting and adaptive filtering to mitigate failure cases is an important direction.\\
\section{Conclusion}
This work revisits language-vision alignment in VLMs from prompt-side semantics. We identify that existing methods for extracting activation maps can suffer from decoding drift, undermining alignment when relying on generated answer tokens.

To address this drift and improve alignment, we propose PV-TAM to leverage prompt tokens as language inputs. 
We further introduce structural denoising filter for subtracting spatial bias. Ultimately, special metrics (TGR, TDR, MinDist) are designed to precisely quantify attention alignment.

Experiments on PartImageNet and RefCOCO demonstrate that PV-TAM significantly outperforms existing baselines in part-level attention, with consistent gains across diverse models (Qwen2-VL, Qwen2.5-VL, InternVL). Notably, prompt-side results always surpass answer-side under Qwen2-VL-7B \cite{wang2024qwen2vl}, suggesting that alignment-relevant attention patterns emerge during prompt encoding rather than answer generation.

We hope this work motivates future research into prompt-centric explanation mechanisms and their applications in model debugging, distillation, and trustworthy multimodal systems.

\clearpage
\section*{Acknowledgements}
Please insert your acknowledgments here.

%
%
\bibliographystyle{splncs04}
\bibliography{main}
\end{document}